\documentclass[journal]{IEEEtran}
\usepackage{booktabs}
\usepackage{multirow}
\usepackage{cite}
\usepackage{url}
\usepackage{amsmath,amssymb,amsfonts}
\usepackage{algorithmic}
\usepackage{graphicx}
\usepackage{textcomp}
\usepackage{xcolor}
\usepackage{amsthm}
\usepackage{adjustbox}

\usepackage{algorithm2e}
\newtheorem{theorem}{Theorem}
\newtheorem{lemma}[theorem]{Lemma}
\RestyleAlgo{ruled}
\usepackage{balance}
\def\BibTeX{{\rm B\kern-.05em{\sc i\kern-.025em b}\kern-.08em
    T\kern-.1667em\lower.7ex\hbox{E}\kern-.125emX}}

\begin{document}

\title{ST-DPGAN: A Privacy-preserving Framework for Spatiotemporal Data Generation
% \thanks{We will public the code in GitHub if this paper gets accepted}
}

\author{Wei~Shao,
        Rongyi~Zhu,
        Cai~Yang,
        Chandra~Thapa,
        Muhammad~Ejaz~Ahmed,
        Seyit~Camtepe,
        Rui~Zhang,
        Du Yong~Kim,
        Hamid~Menouar,
        and~Flora~D.~Salim% <-this % stops a space
\thanks{Wei Shao, Chandra Thapa,  Muhammad Ejaz Ahmed and Seyit~Camtepe are with the Data61, CSIRO, Australia. email: wei.shao@data61.csiro.au.}% <-this % stops a space
\thanks{Rongyi~Zhu is with the University of Rochester. email: rongyi.zhu@rochester.edu}% <-this % stops a space
\thanks{Cai~Yang is with Australian National University, Australia.}% <-this % stops a space
\thanks{Rui~Zhang is with Tsinghua University, China.}% <-this % stops a space
\thanks{Du Yong~Kim is with RMIT University, Australia.}% <-this % stops a space
\thanks{Hamid~Menouar is working at Qatar Mobility Innovations Center, Qatar}% <-this % stops a space
\thanks{Flora~D.~Salim is with University of New South Wales, Australia.}% <-this % stops a space
\thanks{Manuscript received April 19, 2005; revised August 26, 2015.}}

%\author{\IEEEauthorblockN{Author Anonymous}
%\IEEEauthorblockA{Paper ID DM454}
%\and
%\and
%\and
%\and
%\and
\markboth{IEEE INTERNET OF THINGS JOURNAL ,~Vol.~14, No.~8, August~2015} %
{Shao \MakeLowercase{ \textit{et al.}}: ST-DPGAN: A Privacy-preserving Framework for Spatiotemporal Data Generation }

\maketitle

\begin{abstract}

Spatiotemporal data is prevalent in a wide range of edge devices, such as those used in personal communication and financial transactions. Recent advancements have sparked a growing interest in integrating spatiotemporal analysis with large-scale language models. However, spatiotemporal data often contains sensitive information, making it unsuitable for open third-party access. To address this challenge, we propose a Graph-GAN-based model for generating privacy-protected spatiotemporal data. Our approach incorporates spatial and temporal attention blocks in the discriminator and a spatiotemporal deconvolution structure in the generator. These enhancements enable efficient training under Gaussian noise to achieve differential privacy. Extensive experiments conducted on three real-world spatiotemporal datasets validate the efficacy of our model. Our method provides a privacy guarantee while maintaining the data utility. The prediction model trained on our generated data maintains a competitive performance compared to the model trained on the original data.

%Not only does our method generate high-quality spatiotemporal data, but it also effectively preserves its utility within specific privacy budgets. 
%We hope our work will inspire more comprehensive analyses in the broader field of spatiotemporal data.

\end{abstract}

\begin{IEEEkeywords}
differential privacy, spatiotemporal data, generative adversarial network
\end{IEEEkeywords}

\section{Introduction}

Spatiotemporal data permeates various aspects of our lives, ranging from everyday experiences like transportation \cite{ermagun2018spatiotemporal}, website browsing \cite{zhang2009spatio}, and social communication \cite{chae2012spatiotemporal} to larger-scale domains such as criminology, energy transmission \cite{yi2015spatiotemporal}, and currency flow \cite{zhou2022spatiotemporal}. Consequently, spatiotemporal data inherently encompasses sensitive information linked to both time and space. For instance, parking data in the central business districts of Melbourne has its geographic insights~\cite{ijcai2020-463,shao2021fadacs}, which consists of the human mobility pattern, the traffic flow pattern, and the government traffic policies \cite{xue2021mobtcast,xue2021termcast}. Similarly, spatiotemporal data from social media often contains abundant temporal personal information, including browsing history, traveling details, and social network changes \cite{beigi2018privacy}. These factors raise significant privacy concerns regarding the public availability and dissemination of spatiotemporal data.

%spatiotemporal data often contain sensitive information, such as geographic location, traffic flow, and economic activities. 
As mentioned above, semantic information in spatiotemporal data poses a new challenge for fine-grained exploitation. Many regions have proposed data protection regulations due to privacy concerns. The European Union has prompted the General Data Protection Regulation (GDPR), which requires appropriate measures, such as pseudonymisation, to protect data privacy \cite{EUDataBill}. 
% add explnation on social media like TikTok
This regulation makes it important to find a de-identification method when we want to publish more fine-grained spatiotemporal datasets. This technique is also important in other parts of the world.
% 2021, Chinese lawmakers introduced legal provisions to safeguard sensitive data pertaining to geographic information, personnel flow, vehicle flow, and other areas, such as military management zones, national defense research and development units, and county-level party and government agencies \cite{Trans2021}. 
Chinese regulators also encourage conducting anonymous and de-identification processing to the fullest extent possible on data processing activities \cite{DataLaw2021}. In 2021, the United States set up an independent data protection agency to regulate specified high-risk data practices, such as using automated decision systems \cite{USDataBill}. Thus, an answer to how to protect data privacy is important to the further study of spatiotemporal analysis.
%These measures aim to strike a balance between data utilization and privacy protection in the context of spatiotemporal data.

\begin{figure}
  \centering
    \includegraphics[width=0.4\textwidth]{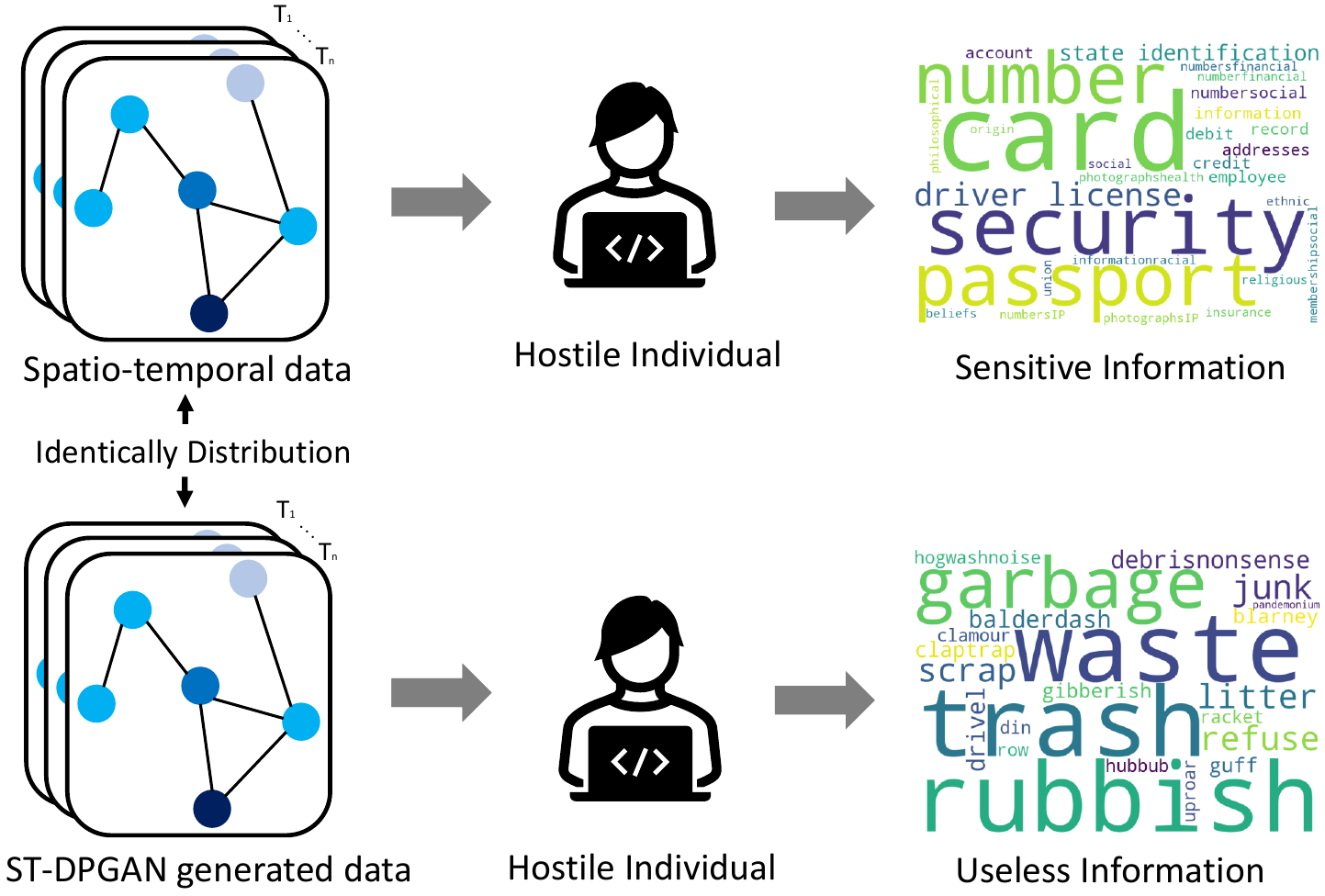}
  %\fbox{\rule[-.5cm]{0cm}{4cm} \rule[-.5cm]{4cm}{0cm}}
  \caption{A hostile individual can extract sensitive information from original spatiotemporal data. Our proposed method enables the generation of privacy-protected data while maintaining good data quality. Our method serves as a fundamental study for applying spatiotemporal analysis techniques on a large scale.}
  \label{fig:teaser}
\end{figure}

%papernot2018scalable, abadi2016deep
%To protect the privacy of spatiotemporal data and allow large-scale public analysis, one of the most popular methods is to apply the differential privacy mechanism to the training process of GANs~\cite{dwork2006calibrating}. Differential privacy (DP) can provide rigorous privacy guarantees to samples in the training data. Xie \textit{et al.} \cite{xie2018differentially} propose DPGAN, which is one of the earliest works for applying differential privacy to GANs. DPGAN has been used to generate images and electronic health record data. Zhang \textit{et al.} \cite{zhang2018differentially} proposed another similar work called dp-GAN, which can also be used to generate image data. These studies have been proposed to train GANs with differential privacy mechanisms to learn the data distributions without leaking sensitive information about each data sample.   
%Most existing differential privacy GAN-based approaches have mostly been applied on image data generation to find an effective loss function to distinguish images. %However, many other types of data also play crucial roles in many research and industrial areas, and they also have the privacy issues.

To enable large-scale spatiotemporal data analysis, we need to provide a privacy guarantee. There has been a wide study on how to bring differential private mechanisms into various data generation schemes, such as GAN\cite{xie2018differentially, jordon2018pate}, VAE\cite{chen2018differentially, weggenmann2022dp} and U-Net\cite{dockhorn2022differentially}. In this work, we focus on incorporating GAN into differential private spatiotemporal data generation. As shown in recent studies \cite{gao2022generative}, GAN has more profound potential in spatiotemporal data generation than other methods.

There are challenges in generating spatiotemporal data using GANs with differential privacy.
%\textbf{(1)} Existing metrics are not capable to measure the difference between two samples in many spatiotemporal data scenarios.  
\textbf{(1)} Spatiotemporal data are heterogeneous. The generated spatial information and temporal information are hard to align into spatiotemporal data as it is demonstrated that spatial and temporal features belong to different Euclidean spaces~\cite{shao2016clustering}. \textbf{(2)} Due to the space-time duality property of spatiotemporal data, generating spatiotemporal data samples with a single variable drawn from a certain distribution is difficult, which is different from homogeneous data such as pixel-based images \cite{salimans2016improved}. \textbf{(3)} Simply adding noises to each spatial location or temporal point cannot guarantee a privacy-preserving dataset for training usage because it does not consider spatial correlation and temporal correlation, and the overall data distribution cannot be maintained.

To overcome the abovementioned challenges, we propose a novel differential private generative adversarial network for generating spatiotemporal data, emphasizing privacy preservation. In our framework, we introduce two key strategies for effectively extracting the distribution of spatiotemporal data and measuring privacy loss during the training process. First, we adopt the Transposed One-Dimensional convolution layer (transConv1d) to generate fake data samples from Gaussian noises, which can transform single-channel Gaussian noises into spatiotemporal data. Second, we simplify existing spatiotemporal graph convolutional networks \cite{yu2017spatio} and introduce temporal attention block and spatial attention block. These blocks are designed to model spatial and temporal dependencies, enabling simultaneous extraction of spatiotemporal features from graph-based time series data and facilitating easier convergence during the training process. The contributions of this paper are summarized as follows:

% modification content
\begin{itemize}
    \item Introduce an end-to-end framework based on Generative Adversarial Networks (GANs) to generate graph-based spatiotemporal data, incorporating a differential privacy mechanism.
    \item Construct a spatiotemporal convolution module capable of converting 1-D Gaussian noise data into 2-D data, enabling the generation of spatiotemporal data from a random variable.
    %\item design a simplified version of the spatiotemporal attention block to effectively align the spatial and temporal information within the generated data. 
    \item Propose to utilize two layers of the spatiotemporal attention block in response to the differential private training on the highly heterogeneous spatiotemporal data.
    %\item evaluates our proposed method on three extensive real-world spatiotemporal datasets and demonstrates its superior performance compared to baseline approaches, particularly in prediction tasks with controlled privacy budgets.
    \item Evaluate our proposed method on three extensive real-world spatiotemporal datasets and finds our methods maintain comparable data utilities with baseline approaches while providing a privacy guarantee. 
\end{itemize}

\section{Related Works}
\noindent \textbf{Differential Private SGD} %The privacy loss of the machine learning process has always been a hot research topic because the model's training needs a huge amount of training data. The privacy of these training data could be compromised due to the machine learning model's strong representation ability. The differential privacy \cite{dwork2006calibrating} has been proposed to measure the privacy loss of a certain data processing mechanism. This theory can be applied in the deep learning area to measure the maximum differential privacy loss $\sigma$ of a certain deep learning process. The Differential Preserving SGD \cite{abadi2016deep} is the first training mechanism that uses Differential Preserving in the training process of deep learning by clipping the training weights and adding noise to the clipped training weights. After that, \cite{abadi2016deep} tried further to improve the computational efficiency of this training mechanism and published the TensorFlow implementation of the improved Differential Preserving SGD mechanism.  
Differential privacy has been introduced into machine learning to secure the training data. The privacy of these training data could be compromised due to the machine learning model's strong representation ability. The differential privacy \cite{dwork2006calibrating} has been proposed to measure the privacy loss of a certain data processing mechanism. This theory can be applied in the deep learning area to measure the maximum differential privacy loss $\sigma$ of a certain deep learning process. The Differential Preserving SGD \cite{abadi2016deep} is the first training mechanism that uses Differential Preserving in the training process of deep learning by clipping the training weights and adding noise to the clipped training weights. After that, \cite{abadi2016deep} tried further to improve the computational efficiency of this training mechanism and published the TensorFlow implementation of the improved Differential Preserving SGD mechanism. However, these works focus on the general training procedure while ignoring the differences in various kinds of data.

\textbf{Private Aggregation of Teacher Ensembles (PATE)}  Besides the Differential preserving SGD, there is another way to protect the privacy of the deep learning model. The Private Aggregation of Teacher Ensembles (PATE)~\cite{papernot2016semi} uses several teacher models, which are trained on different partitions of the privacy data, and a student model, which is trained on unlabelled public data. The label of the public data is obtained by querying the trained teacher models. The privacy-preserving and utility of the whole model can be further enhanced by making the teachers reach a consensus on query voting. 

However, in many real applications \cite{gao2022generative}, such as currency flow and social analysis, public data could be difficult to obtain. In our application, our model generates spatiotemporal data where the public spatial-temporal dataset is difficult to obtain. So, we adopted the DP-SGD structure and applied it in our discriminator to protect the privacy of the generated spatial-temporal data. %Please check the modified sentence to see if it is acceptable.

%For those two methods that apply differential privacy in deep learning, each of them has been further developed to be applied in GAN to preserve the privacy of generated data. %I do not understand what it means.%
%Xu \textit{et al.} \cite{xu2019ganobfuscator} transformed the differential preserving SGD and applied that in GAN by clipping and adding noise to the training weights in the discriminator. Similarly, \cite{jordon2018pate} transformed the PATE structure and applied that in GAN by transforming the discriminators into teacher models, and transforms the generator into the student model. Then, the training weights of the teachers are voted to the generator to preserve the privacy.  

\section{Preliminary}
In this section, we briefly discuss how privacy in our proposed method is protected in deep learning.

%Differential privacy \cite{dwork2006calibrating} is a mechanism proposed to import randomness into gradient updates for individual privacy. 
Differential privacy is a formal mathematical framework that enables the release of aggregate statistical information about a dataset while provably bounding the privacy risk to individual data subjects\cite{dwork2006calibrating}.
In a deep learning scenario, differential privacy is proposed to import randomness into gradient updates for individual privacy. Furthermore, it
%
%In a deep learning scenario, differential privacy 
prevents the membership inference attack. Attackers can infer the input data by observing the output change from a model or a function by changing a record in the input. Differential privacy guarantees that altering one sample in the dataset will not significantly change the distribution of one function's output. Differential privacy under the Gaussian mechanism uses Kullback Leibler (KL) divergence to measure how privacy is protected.

\begin{equation}
    \text{Pr}[\mathcal{M}(x) \in S ] \leq e^\epsilon ~ \text{Pr}[\mathcal{M}(x') \in S ] + \delta,
\label{dpeq}
\end{equation}
where the $\mathcal{M}(\cdot)$ is a differential privacy mechanism, $\epsilon$ is the privacy cost, and the $\delta$ is a relaxation for privacy constraint. 

In this paper, we utilize the DPSGD developed by Google \cite{abadi2016deep}. The basic outline of DPSGD algorithms is as follows:
\begin{itemize}
    \item  Take a random sample $\mathrm{L}$ with sampling probability $\mathrm{L}/N$, where $N$ is the size of the input data.
    \item  For each $i \in \mathrm{L}$, we calculate the gradient $\mathbf{g}\left(x_{i}\right) \leftarrow \nabla_{\theta}\mathcal{L}\left(\theta, x_{i}\right)$, where $\mathcal{L}\left(\theta, x_{i}\right)$ denotes the loss function and $x_i$ is a training sample.
    \item Then we clip the gradient as $\overline{\mathbf{g}}\left(x_{i}\right) \leftarrow \mathbf{g}\left(x_{i}\right) / \max \left(1, \frac{\left\|\mathbf{g}\left(x_{i}\right)\right\|_{2}}{C}\right)$, where $C$ is a clipping bound.
    \item The Gaussian noise $\left.\mathcal{N}\left(0, \sigma^{2} C^{2} \mathbf{I}\right)\right)$ can then be added to the clipping gradient.
    \item Update the model and calculate the privacy loss until the privacy budget is exhausted.
\end{itemize}

DPSGD uses moment accountants to calculate the privacy costs over training iterations. The DPSGD algorithm is $(\epsilon,\delta)$-differential private, in which the standard deviation $\sigma$ of Gaussian Noise is larger than $\Omega(q\sqrt{T\log(1/\delta)}/\epsilon)$, where $q$ is the sampling probability, $\Omega$ is the lower bound, and $T$ is the number of steps.

%In the context of GANs, DPSGD provides a promise to enable accurate learning of the data distribution, despite adding or removing any training sample in the model. since the generator of GANs cannot access the real data, the generator trained with discriminator with privacy is also differentially private. 

\section{Methodology}
\begin{figure*}
  \centering
    \includegraphics[width=0.95\textwidth]{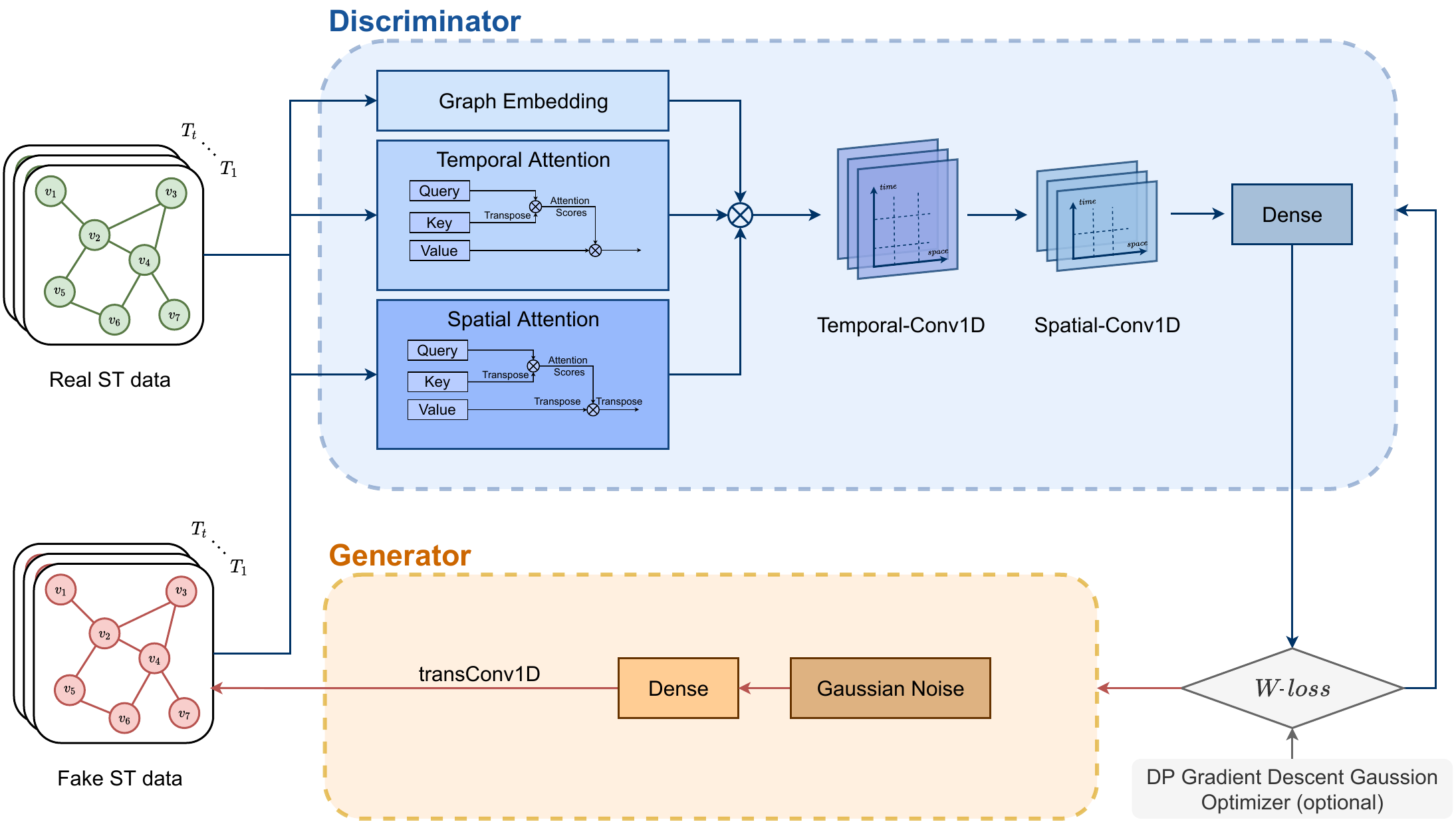}
  %\fbox{\rule[-.5cm]{0cm}{4cm} \rule[-.5cm]{4cm}{0cm}}
  \caption{Architecture of ST-DPGAN. }
  \label{fig:ST-DPGAN}
\end{figure*}

%In this section, we elaborate on the proposed architecture of spatiotemporal Wasserstein generative adversarial networks (ST-WGAN). As shown in Figure \ref{fig:ST-DPGAN}, ST-WGAN is mainly composed of two components: generator and discriminator. The input of the model is real spatiotemporal data represented as a graph  $\mathcal{G}(t) = (\mathbf{V}, \mathbf{E}, \mathbf{X}(t))$, where $\mathbf{V}$ indicates nodes in the graph, $\mathbf{E}$ represents edges connected nodes, and $\mathbf{X}(t)$ denotes attributes of each node associated with temporal information. We apply the Wasserstein distance as the loss function to measure the difference between real training samples and fake generated samples since the training of GAN is unstable and Wasserstein distance has been proved to be an effective loss function to measure the real data and fake data in vision area \cite{arjovsky2017wasserstein}. 

%In the case where privacy of generated data becomes a major concern, we introduce the method used in DPSGD \cite{abadi2016deep} to inject the Gaussian noise into the discriminator during the training procedure. To avoid confusion in the later parts, we refer to the name ST-DPGAN as the DP-version of our proposed model. Once the model is well-trained, the generator can be used to produce simulated spatiotemporal data for training purposes.
%and we apply 11 different popular regressors to test the utility of the generated data. Details can be found in the experiments.

To devise an effective method for generating spatiotemporal data while obfuscating sensitive information, we propose the ST-DPGAN (spatiotemporal Differential-Privacy Generative Adversarial Network). This approach is inspired by the WGAN framework proposed by Arjovsky et al.~\cite{arjovsky2017wasserstein}. Our approach incorporates a traditional GAN structure consisting of two key components: a generator and a discriminator. However, unlike the original WGAN, which primarily focuses on image generation, our primary objective is to generate privacy-preserving spatiotemporal data. This goal necessitates overcoming the heterogeneous and space-time duality of spatiotemporal data. In this section, we elaborate on the strategies we have developed to tackle these challenges. %This goal necessitates overcoming the challenges associated with the properties of spatiotemporal data. In this section, we elaborate on the strategies we have developed to tackle these challenges.

In essence, our proposed method, ST-DPGAN, is an end-to-end framework. The model accepts real spatiotemporal data as input, represented as $\mathcal{G}(t) = (\mathbf{V}, \mathbf{E}, \mathbf{X}(t))$, where $\mathbf{V}$ denotes the nodes in the graph, $\mathbf{E}$ represents the edges connecting the nodes, and $\mathbf{X}(t)$ signifies the attributes associated with each node, incorporating temporal information. The output of our model is privacy-protected spatiotemporal data that maintains a high data quality.

Similar to most GAN architectures, our proposed network, ST-DPGAN, comprises two primary components: a generator and a discriminator. In our framework, we have introduced two novel features to enhance the performance of these components. Firstly, for the generator, we have incorporated a transConv1d module. This module is designed to transform 1-D Gaussian noise data into 2-D spatiotemporal data, thereby improving the synthesis results. Secondly, for the discriminator, we have replaced the conventional three-block convolutional structure with a two-block attention structure. This modification is intended to align spatial and temporal information within the data more effectively. \footnote{We set $\epsilon = \infty$ when Differential Privacy (DP) is not utilized, . Similarly, when self-attention is not employed, we set $\alpha = \beta = 0$ throughout the training process.} The main pipeline of our approach can be summarized with the following pseudo-code notation:

\begin{algorithm}
\caption{Spatiotemporal Differential-Privacy Generative Adversarial Network}
\KwIn{Spatiotemporal Dataset $D$ with $N$ samples, Laplacian matrix $L$, sampling probability $q$, Differential private optimizer $O_{(\epsilon,\delta)}$, Learning rate $l$, noise level $\sigma$, Overall iteration $T$, coefficients $\alpha$, $\beta$ }
Initialize generator parameters $\theta$, which contains $W_gz$, $K_g$, and $b_g$, discriminator parameters $w$, which contains $K_d$ and $b_d$. \\
\textbf{Generator setup}: 
\begin{equation}
% \hspace*{-7cm}
    \begin{aligned}
        &\tilde{z} = W_gz + b_g \\
        &G_{\theta}(z) = \tilde{z}\ast K_g
    \end{aligned} \notag
\end{equation}
\textbf{Discriminator setup}:
\begin{equation}
% \hspace*{0.5cm}
    \begin{aligned}
        &\tilde{x} =  GraphEmb(x, L) + \alpha \cdot SpatialAttn( x ) \\
        & \qquad  +  \beta \cdot TemporalAttn ( x ) \\
        &\bar{x} = \tilde{x}\ast K_d \\
        &F_w(x) = W_d\bar{x} + b_d 
    \end{aligned} \notag
\end{equation}
\For{i = 1,..., T}{
    \For {t = 1, ..., $k_{\text{critic}}$}{
        % \Comment{$f_{w}$ set up by the discriminator in figure 1}
        Take a random sample $D_q$ from $D$ by probability $q$ \\
        Sampling $m$ example $\{z^{(i)}\}_{k=1}^{m} $ from noise prior $ p(z)$, where $m = qN$.  \\
        $g_w \leftarrow \nabla_{w} [ \frac{1}{m} \sum_{k=1}^{m} F_{w}(x^{(i)}) - \frac{1}{m} \sum_{k=1}^{m} F_{w}(G_{\theta}(z^{(i)}) ) ] $  \\
        $w \leftarrow w + l \cdot O_{(\epsilon, \delta)}(g_w)$
    }
    Sampling m example $\{z^{(i)}\}_{k=1}^{m} $ from noise prior $ p(z)$. \\
    $g_{\theta} \leftarrow \nabla_{w} - \frac{1}{m} \sum_{k=1}^{m} F_{w}(G_{\theta}(z^{(i)}) ) $ \\
    $\theta \leftarrow \theta - l \cdot O_{(\epsilon, \delta)}(g_{\theta})$   
}

\label{algo:stwgan}
\end{algorithm}

\subsection{Generator}
To address the challenge of synthesizing spatiotemporal data from 1-D Gaussian noises, we introduce a novel module in our GAN's generator, termed \textit{transConv1d}. In a traditional GAN framework, the generator's role is to learn an approximation of the real data distribution. This is achieved by mapping random noise to data samples, typically sampled from a Gaussian distribution. However, spatiotemporal data presents a unique complexity due to the inherent spatial and temporal correlations. These correlations pose a challenge in reconstructing the data distribution, particularly while maintaining a balance between privacy preservation and data quality.

To overcome this challenge, we designed the transConv1d module to recover spatial and temporal information from a learned distribution. Specifically, at each temporal data location, the data is multiplied by a kernel initialized with a Gaussian distribution to generate new sequential data. All the generated sequential data is then aggregated to form the final output of the transConv1d module. In our model, the filter parameter of the transConv1d module is set to the size of the spatial locations (N), and the kernel size is set to the length of the time series (T). Consequently, the initial random data with dimensions of $1 \times N$ can be processed into dimensions of $T \times N$. This enables the model to capture the temporal variations in the generated data effectively. Our proposed transConv1d module can help the model recover the aligned spatiotemporal information from noise, which we will validate in the next section. 

%[ToDo] The detailed implementation of this module will coming in our github repo[link]

%In our model, the generator consists of a Gaussian noise generator, a fully-connected layer (dense layer) and a transposed One-Dimensional Convolution layer (transConv1d). 

%The Gaussian noise generator aims to produce a batch of data with $\left.\mathcal{N}\left(0, \sigma^{2} C^{2} \mathbf{I}\right)\right)$,  where $\sigma$ and $C$ are set at the beginning to control the noise level and training process. The dense layer is used to expand spatial dimension of the random noise. If the size of the dense layer is $l_n$, the output of the dense layer is a 2D ($B \times l_n$) Gaussian noise, where $B$ indicates the size of a batch. The expanded random noises are then fed into a transConv1d to extend its temporal dimension. 

\subsubsection{transConv1d}
In this section, we provide our mathematical insight into transConv1d to validate the property of the proposed module. 
\begin{theorem}
\label{thm1}
Let the initial random data be $X$ of size $1 \times N$, where each entry is $x_{j}$, and let the kernel be $K$ of size $s\times s$, where each entry is $k_{ij}$. Assume $x_{j}\stackrel{i.i.d}{\sim} \mathcal{N}(\mu_1,\sigma_1^2) $,  $k_{ij}\stackrel{i.i.d}{\sim} \mathcal{N}(\mu_2,\sigma_2^2) $, and $x_{j}$ is mutually independent to $k_{ij}$. Then the output matrix is $Y$ of size $T \times N$, where each entry of its vectorization form $vec(Y)_{i}$ satisfies

\resizebox{1.1\linewidth}{!}{
  \begin{minipage}{\linewidth}
    \begin{align*}
    vec(Y)_{i}&=\frac{\sigma_1^2+\sigma_2^2}{4}(Q_1-Q_2)+G+m\mu_1\mu_2,
\end{align*}
\vspace{.01cm}
  \end{minipage}
}
where $Q_1,Q_2\sim \chi_m^2$, $G\sim \mathcal{N}(0,\frac{m^2(\sigma_1^2+\sigma_2^2)(\mu_1^2+\mu_2^2)}{2})$, and $m$ is the number of nonzero entries in the $i^th$ row of the transposed convolution matrix.
\end{theorem}

\begin{proof}
Let the output matrix be $Y$, then by  transposed convolution, the vectorization form of $Y$, $vec(Y)_{i}$,  is given by,
\resizebox{.85\linewidth}{!}{
  \begin{minipage}{\linewidth}
    \begin{align}
    vec(Y)_{i}=\sum_{j=1}^{N} k_{ij}\cdot x_{j}
\end{align}
  \end{minipage}
}
\vspace{.15cm}

where $vec(Y)_{i}$ is the $i^{th}$ entry of the vectorization form of $Y$, $k_{ij}$ is the $(i, j)^{th}$ entry of the kernel $K$, and $x_j$ is the $j^{th}$ entry of the data matrix $X.$ Hence, $x_{j}{\sim} \mathcal{N}(\mu_1,\sigma_1^2) $ and $k_{ij}\stackrel{i.i.d}{\sim} \mathcal{N}(\mu_2,\sigma_2^2) $. Note that $k_{ij} x_{j}$ is given by:
\resizebox{.85\linewidth}{!}{
  \begin{minipage}{\linewidth}
    \begin{align}
    k_{ij}\cdot x_{j}&=\frac{1}{4}(k_{ij}+ x_{j})^2-\frac{1}{4}(k_{ij}- x_{j})^2,
\end{align}
  \end{minipage}
}
\vspace{.15cm}

where $k_{ij}+ x_{j}\sim\mathcal{N}(\mu_1+\mu_2,\sigma_1^2+\sigma_2^2)$ and $k_{ij}- x_{j}\sim\mathcal{N}(\mu_2-\mu_1,\sigma_1^2+\sigma_2^2)$ since $k_{ij}$ and $x_j$ are independent. Let 
%$a=\mu_1+\mu_2$, $b=\mu_2-\mu_1$, $c=\sigma_1^2+\sigma_2^2$, and
\resizebox{.85\linewidth}{!}{
  \begin{minipage}{\linewidth}
    \begin{align}
    Z_1&=\frac{1}{\sqrt{\sigma_1^2+\sigma_2^2}}(k_{ij}+ x_{j})-\frac{\mu_1+\mu_2}{\sqrt{\sigma_1^2+\sigma_2^2}} ,\mbox{ and}\\
    Z_2&=\frac{1}{\sqrt{\sigma_1^2+\sigma_2^2}}(k_{ij}- x_{j})-\frac{\mu_2-\mu_1}{\sqrt{\sigma_1^2+\sigma_2^2}}.
\end{align}
  \end{minipage}
}
\vspace{.15cm}
%$Z_1=\frac{1}{\sqrt{\sigma_1^2+\sigma_2^2}}(k_{ij}+ x_{j})-\frac{\mu_1+\mu_2}{\sqrt{\sigma_1^2+\sigma_2^2}}$ and $Z_2=\frac{1}{\sqrt{\sigma_1^2+\sigma_2^2}}(k_{ij}- x_{j})-\frac{\mu_2-\mu_1}{\sqrt{\sigma_1^2+\sigma_2^2}}$ and 

It is evident that $Z_1,Z_2\sim \mathcal{N}(0,1)$. Hence, we have the following,

\resizebox{.85\linewidth}{!}{
  \begin{minipage}{\linewidth}
    \begin{align}
    k_{ij} x_{j}&=\frac{1}{4}(\sqrt{\sigma_1^2+\sigma_2^2}Z_1+\mu_1+\mu_2)^2-\frac{1}{4}(\sqrt{\sigma_1^2+\sigma_2^2}Z_2+\mu_2-\mu_1)^2\\
    &=\frac{\sigma_1^2+\sigma_2^2}{4}(Z_1^2-Z_2^2)+P+\mu_1\mu_2,
\end{align}
  \end{minipage}
}
\vspace{.15cm}

where $Z_1^2,Z_2^2\sim\chi_1^2$, and $P$ is given by

\resizebox{.85\linewidth}{!}{
  \begin{minipage}{\linewidth}
    \begin{align}
    P&=\frac{\sqrt{\sigma_1^2+\sigma_2^2}(\mu_1+\mu_2)}{2}Z_1-\frac{\sqrt{\sigma_1^2+\sigma_2^2}(\mu_2-\mu_1)}{2}Z_2\\
    P&\sim\mathcal{N}(0,\frac{(\sigma_1^2+\sigma_2^2)(\mu_1^2+\mu_2^2)}{2}),
\end{align}
  \end{minipage}
}
\vspace{.15cm}

since $Z_1$ and $Z_2$ are also independent.
Note that the number of addends in (4) is the number of nonzero entries in the %$\floor{\frac{i}{s}}^{th}$ 
$i^{th}$ row of the transposed convolution matrix. Suppose there are $m$ nonzero entries in the $i^{th}$ row of the transposed convolution matrix. Therefore, it follows that 

\resizebox{.85\linewidth}{!}{
  \begin{minipage}{\linewidth}
    \begin{align}
    vec(Y)_{i}&=\sum k_{ij}\cdot x_{j}\\
    &=\frac{\sigma_1^2+\sigma_2^2}{4}(Q_1-Q_2)+R+m\mu_1\mu_2,
\end{align}
  \end{minipage}
}
\vspace{.15cm}

where \begin{equation}
    Q_1,Q_2\sim \chi_m^2, R\sim \mathcal{N}(0,\frac{m^2(\sigma_1^2+\sigma_2^2)(\mu_1^2+\mu_2^2)}{2})
\end{equation}
\end{proof}

Considering the actual training, we specifically let the mean of the kernel be approximately the same as the mean of the training data. Hence, we provide two useful results below about the common setting and the setting we used.
\begin{lemma}
Using the same setting as Theorem 1:
\begin{enumerate}
    \item if $x_{j}\stackrel{i.i.d}{\sim} \mathcal{N}(0,1) $ and  $k_{ij}\stackrel{i.i.d}{\sim} \mathcal{N}(0,1) $, then $vec(Y)_{i}$ is the difference of two chi-squared random variables with scaling.
    
    \resizebox{1.1\linewidth}{!}{
  \begin{minipage}{\linewidth}
    \begin{align*}
    vec(Y)_{j}
        &=\frac{1}{2}(Q_1-Q_2)
\end{align*}
  \end{minipage}
}
\vspace{.15cm}
where $Q_1,Q_2\sim \chi_m^2$.
    
    \item  if $x_{j}\stackrel{i.i.d}{\sim} \mathcal{N}(0,1) $ and  $K_{ij}\stackrel{i.i.d}{\sim} \mathcal{N}(\mu,1) $, where $\mu$ is the mean of the training data, then 
    
    \resizebox{1.1\linewidth}{!}{
  \begin{minipage}{\linewidth}
    \begin{align*}
    vec(Y)_{i}
        &=\frac{1}{2}(Q_1-Q_2)+G
\end{align*}
  \end{minipage}
}
\vspace{.15cm}

where $Q_1,Q_2\sim \chi_m^2$ and $G\sim\mathcal{N}(0,m^2\mu^2)$

\end{enumerate}
\end{lemma}
From the above explanation, we can see that transConv1d has the property to recover the distribution for both spatial and temporal information. We will discuss and validate this claim in our experimental session. 

%The transConv1d is a trainable layer that could reverse the process of the traditional CNN1D, and it is used in the generator to learn the temporal distribution of data. To be specific, in each temporal data location, the data is multiplied by a kernel initialised with Gaussian distribution to generate new sequential data. All the generated sequential data is added together to form the final transConv1d processed data. In our model, the filter parameter is set to be the size of the locations N, and the size of the kernel is set to be the length of the time series T . Therefore, the initial random data of dimension $1 \times N$ could be processed into the dimension of $T \times N$. By doing this, the model can learn the temporal variations of the generated data. transConv1d plays a key role in spatiotemporal data generation. Most existing GAN-based model are used to produce the image data by applying a normal 2D convolution layer. However, we find that generating a normal 2D layer of Gaussian noises to form spatiotemporal graph data performs worse than applying 1D convolutional layer to temporal dimension and spatial dimension, respectively. We validate this in our experimental session. 

\begin{table*}[!hbt]
\caption{We show the MSE and MAE scores of predictive models. Our models are marked by \textbf{*}. The lower the score value, the better the data quality. We choose the privacy budget $\epsilon$ as 12 for the DP-based model to compare with other non-DP contained algorithms. The Average row stands for the average loss under different prediction algorithms. The Count rows represent the total number of best results for each column.}
\begin{adjustbox}{width=1.0\textwidth, center}
    \centering
    \begin{tabular}{c|c|| c c || c c | c c | c c }  
    \toprule
    \multicolumn{2}{c||}{Methods} &  \multicolumn{2}{c||}{Real}  &  \multicolumn{2}{c|}{WGAN} & \multicolumn{2}{c}{ST-DPGAN$^*$}  & \multicolumn{2}{|c}{ST-DPGAN-Attn$^*$} \\
    \midrule 
    \multicolumn{2}{c||}{Metrics} & MSE & MAE & MSE & MAE & MSE & MAE  & MSE & MAE \\
    \midrule\midrule 
    \multirow{11}{*}{ \rotatebox[origin=c]{90}{Parking} } 
    & LR                & 0.0004 & 0.0158    & \textbf{0.0068} & \textbf{0.0608}       & 0.0074 & 0.0678     & \textbf{0.0068} & 0.0643 \\
    & Linear SVR        & 0.0003 & 0.0141    & 0.0069 & 0.0605   & 0.0066 & 0.0594     & \textbf{0.0062} & \textbf{0.0576} \\
    & MLP               & 0.0005 & 0.0188    & 0.0091 & 0.0743   & 0.0077 & 0.0677     & \textbf{0.0068} & \textbf{0.0645} \\
    & SGD               & 0.0004 & 0.0158    & \textbf{0.0068}   & \textbf{0.0608}     & 0.0076 & 0.0686 & 0.0069          & 0.0646 \\
    & Decision Tree     & 0.0046 & 0.0479    & 0.0218 & 0.1127   & 0.0150 & 0.0877     & \textbf{0.0136} & \textbf{0.0828} \\
    & Random Forest     & 0.0008 & 0.0218    & 0.0132 & 0.0887   & 0.0077 & 0.0667     & \textbf{0.0069} & \textbf{0.0613} \\
    & Gradient Boosting & 0.0004 & 0.0166    & 0.0117 & 0.0838   & 0.0074 & 0.0676     & \textbf{0.0071} & \textbf{0.0630} \\
    & Bagging           & 0.0012 & 0.0255    & 0.0148 & 0.0927   & 0.0084 & 0.0691     & \textbf{0.0076} & \textbf{0.0638} \\
    & Ada Boosting      & 0.0027 & 0.0451    & 0.0134 & 0.0919   & 0.0100 & 0.0804     & \textbf{0.0096} & \textbf{0.0774} \\
    & LGBM              & 0.0003 & 0.0157    & 0.0103 & 0.0791   & 0.0074 & 0.0673     & \textbf{0.0065} & \textbf{0.0603} \\
    & LSTM              & 0.0118 & 0.0920    & 0.0162 & 0.1035   & \textbf{0.0087} & 0.0714    & \textbf{0.0087} & \textbf{0.0689} \\
    
    \midrule
    \multicolumn{2}{c||}{Average} & 0.0021 & 0.0299 & 0.0119 & 0.0826 & 0.0085  & 0.0703 & 0.0078  & 0.0662  \\
    
    \midrule\midrule
    \multirow{11}{*}{ \rotatebox[origin=c]{90}{METR-LA} } 
    & LR                & 91.8794 & 6.7027      & \textbf{95.2423}  & 6.8065     & 97.1440  & 6.8816          & 97.0645            & \textbf{6.7750}  \\
    & Linear SVR        & 89.4740 & 6.3760      & \textbf{96.3206}  & 6.7831     & 97.7706  & \textbf{6.6635} & 97.2597            & 6.6976  \\
    & MLP               & 103.0017 & 6.9999     & 99.0773           & 7.4096     & 110.6378 & 7.2572          & \textbf{97.6307}   & \textbf{6.8797}  \\
    & SGD               & 91.8495 & 6.6269      & \textbf{95.7337}  & 6.8376     & 96.8873  & 6.8355          & 97.0175            & \textbf{6.7621}  \\
    & Decision Tree     & 225.9751 & 10.1180    & 286.1651          & 15.0811    & 191.0109 & 9.7920          & \textbf{142.6260}  & \textbf{8.0813}  \\
    & Random Forest     & 119.4058 & 7.6720     & 294.5921          & 15.4848    & 111.4184 & 7.5598          & \textbf{108.7084}  & \textbf{6.8320}  \\
    & Gradient Boosting & 101.2091 & 7.0300     & 231.6227          & 13.8163    & 105.7373 & 7.1437          & \textbf{102.2363}  & \textbf{6.7753}  \\
    & Bagging           & 129.1552 & 7.9630     & 176.2332          & 11.9194    & 119.4350 & 7.8397          & \textbf{113.0507}  & \textbf{7.0063}  \\
    & Ada Boosting      & 159.0458 & 10.7724    & 244.9048          & 14.1454    & 141.7784 & 10.0757         & \textbf{107.5073}  & \textbf{8.0823} \\
    & LGBM              & 101.7017 & 7.0309     & 241.4298          & 14.0983    & 102.7747 & 7.0127          & \textbf{102.0630}  & \textbf{6.9834}  \\
    & LSTM              & 123.6026 & 8.1991     & 151.5952 & 10.3412    & 126.8095 & 8.2104    & \textbf{122.6244}  & \textbf{8.0745} \\
    
    \midrule  
    \multicolumn{2}{c||}{Average} & 121.4818 & 7.7719   & 182.9924 & 11.1567   & 118.3094 & 7.7520   & 107.9808 & 7.1772 \\
    \midrule\midrule
    
    \multirow{11}{*}{ \rotatebox[origin=c]{90}{Windmill} } 
    & LR                & 0.0496 & 0.1739     & 0.0607 & \textbf{0.1885}     & 0.0609 & 0.2130     & \textbf{0.0581} & 0.2022 \\
    & Linear SVR        & 0.0615 & 0.1620     & 0.0603 & 0.1968     & \textbf{0.0570} & 0.1778     & 0.0771 & \textbf{0.1751} \\
    & MLP               & 0.0514 & 0.1752     & 0.0920 & 0.2633     & 0.0609 & 0.2132     & \textbf{0.0578} & \textbf{0.2011} \\
    & SGD               & 0.0498 & 0.1737     & 0.0607 & 0.1884     & 0.0613 & 0.2140     & \textbf{0.0582} & \textbf{0.2028} \\
    & Decision Tree     & 0.1085 & 0.2356     & 0.0828 & 0.2493     & 0.1393 & 0.2817     & \textbf{0.1353} & \textbf{0.2702} \\
    & Random Forest     & 0.0530 & 0.1805     & 0.0814 & 0.2514     & \textbf{0.0629} & 0.2178     & 0.0637 & \textbf{0.2086} \\
    & Gradient Boosting & 0.0509 & 0.1746     & 0.0734 & 0.2404     & 0.0623 & 0.2173     & \textbf{0.0590} & \textbf{0.2063} \\
    & Bagging           & 0.0580 & 0.1863     & 0.0835 & 0.2552     & 0.0792 & 0.2339     & \textbf{0.0699} & \textbf{0.2137} \\
    & Ada Boosting      & 0.0576 & 0.2085     & 0.0668 & 0.2235     & 0.0743 & 0.2455     & \textbf{0.0604} & \textbf{0.2075} \\
    & LGBM              & 0.0511 & 0.1750     & 0.0722 & 0.2395     & 0.0621 & 0.2168     & \textbf{0.0592} & \textbf{0.2069} \\
    & LSTM              & 0.0568 & 0.1864     & \textbf{0.0603} & \textbf{0.1834}     & 0.0897 & 0.2560     & 0.0991 & 0.2703 \\
     
    \midrule
    \multicolumn{2}{c||}{Average} & 0.0589 & 0.1847 & 0.0722  & 0.2254  & 0.0736  & 0.2261 & 0.0725  & 0.2150 \\
    \midrule\midrule
    \multicolumn{2}{c||}{Count} & - & - & 6 & 4 & 3 & 2 & 26 & 28 \\
    \bottomrule
    
    \end{tabular}
\end{adjustbox}
\label{tab:exp-main}
\end{table*}

\subsection{Discriminator}

%To extract the spatiotemporal information from real data and fake data, we form the ST-graph data into a normalised Laplacian matrix and apply the 1D convolution operation on the temporal and spatial dimension respectively.

% discriminator aims to distinguish real data and fake data produced by the generator. Graph embedding approach is widely used in graph neural networks to extract the correlation in spatial domains. We form the ST-graph data into a normalised Laplacian matrix to extract the spatiotemporal information from real data and fake data. 

% With the extracted graph-based features, we need to further explore the hidden spatiotemporal information and try to align the spatial and temporal regularities. 

The conventional pipeline for discriminating spatiotemporal data, known as the spatiotemporal Graph Convolutional Network (STGCN) \cite{yu2017spatio}, typically employs three convolutional blocks to predict the temporal evolution of spatiotemporal data. However, in our research, we have observed that the original STGCN structure is not ideally suited as a discriminator in a GAN framework. This is primarily due to its heavy three-layer spatiotemporal structure, which can lead to computational issues such as gradient explosion and vanishing. These challenges necessitate reconsidering the discriminator design in the context of GANs for spatiotemporal data.

To address this problem, we propose the integration of two self-attention blocks within the discriminator as an enhancement to the "basic" model. Firstly, while convolutional layers are commonly used in modern GAN architectures for feature extraction, they have limitations in capturing global dependencies between nodes, which are crucial in spatiotemporal data. We believe that each node is influenced by other nodes, both spatially and temporally. Therefore, we introduce attention structures to better capture and align the spatiotemporal information. The spatial self-attention block establishes relationships between nodes by attending to every other node simultaneously. Similarly, the temporal self-attention block links each node by attending to itself across all timestamps.

The feature extraction process is formulated as follows: The input data is transformed through the normalized Laplacian matrix, as shown in Equation (1), where $X_{l} \in \mathbb{R}^{N \times T}$. Equation (2) describes the process of spatial self-attention. The input $X \in \mathbb{R}^{T \times N}$ is first projected using $W_{s}^{q}$ and $W_{s}^{k} \in \mathbb{R}^{T \times M}$ to obtain query and key matrices, where $M$ is the hidden dimension. The value matrix $V_{s} \in \mathbb{R}^{N \times T}$ is obtained through 1D convolution with a kernel size 1. The raw attention scores are then passed through a softmax function and multiplied with $V_{s}$ to produce the new representation $X_{s} \in \mathbb{R}^{N \times T}$. The process is similar for temporal self-attention, where $W_{t}^{q}$ and $W_{t}^{k} \in \mathbb{R}^{T \times H}$ are the projection matrices for query and key. The value matrix $V_{t} \in \mathbb{R}^{T \times N}$ is multiplied with the attention scores to obtain the new representation $X_{t} \in \mathbb{R}^{T \times N}$.

We consider graph embedding as an initial means of obtaining features. Following a similar approach to Zhang et al. \cite{zhang2019self}, the output from the two self-attention blocks is multiplied by scalars initialized with a value of 0 and added back to the basic features. This ensures the model begins with fundamental features derived from the graph properties and gradually learns other underlying features during training. The final representation is expressed as $Z = X_{l} + \alpha X_{s} + \beta X_{t}^\intercal$. To avoid confusion, we denote models with the additional self-attention blocks as ST-DPGAN-Attn and empirically validate the performance improvements brought by attention in our experiments.

\begin{equation}
    X_{l} = LX^\intercal
    \label{eq:laplacian}
\end{equation}

\begin{equation}
    \begin{aligned}\label{eq:temporal}
        Q_{s} &= X^\intercal W_{s}^{q},\enspace K_{s} = X^\intercal W_{s}^{k},\enspace V_{s} = ( F_{s}\ast G_{s})(X^\intercal )\\
        X_{s} &=  \sigma(Q_{s}K_{s}^\intercal)V_{s}
    \end{aligned}
\end{equation}

\begin{equation}
    \begin{aligned}\label{eq:spatial}
        Q_{t} &= XW_{t}^{q},\enspace K_{t} = XW_{t}^{k},\enspace V_{t} = (F_{t}\ast G_{t})(X)\\
        X_{t} &=  \sigma(Q_{t}K_{t}^\intercal)V_{t}^\intercal
    \end{aligned}
\end{equation}

Besides that, for the optimizer in the presence of differential privacy, DPSGD is used in the discriminator to add random noise in the gradients of the trainable parameters through back-propagation \cite{abadi2016deep}.

\subsection{Model Summary}
We now summarise the main characteristics of our model as follows:
\begin{itemize}
    %\item ST-DPGAN is an end-to-end GAN-based framework to generate graph-based spatiotemporal data with quantitative privacy-preserving guaranteed by differential privacy mechanism. The generated data could be used for not only regression tasks but also other machine learning or deep learning applications.
    \item ST-DPGAN is an end-to-end GAN-based framework to generate spatiotemporal data with quantitative privacy-preserving guaranteed by the differential privacy mechanism. The generated data could be used not only for regression tasks but also for other machine learning or deep learning applications.
    \item The TransConv1D focuses on transposing the spatiotemporal relation from one-dimensional Gaussian noise, which controls the initial distribution of the simulation dataset.
    %\item A spatial attention blocks and a temporal attention  that aims to extract the spatial and temporal features and align both spatial and temporal regularities. 
    \item The spatial and temporal blocks help extract underlying spatial and temporal relations among nodes, which further boosts the performance of the model in downstream tasks.
\end{itemize}

\section{Experiments}
In this section, we aim to present the experiments conducted on our ST-DPGAN model. We provide a concise overview of the experimental setup, including details about the datasets employed and the chosen evaluation metrics. Furthermore, we evaluate our methods in relation to data quality and privacy budgets. Subsequently, we assess the variance of our approach within this section. The results obtained from our experiments not only surpass those of alternative methods but also serve as validation for the efficacy of our proposed techniques.

\subsection{Dataset}
We conduct extensive experiments on three publicly available large-scale real-world datasets. \textbf{Melbourne parking dataset} \cite{ijcai2020-463,shao2021fadacs}, collected by sensors among parking bays; \textbf{METR-LA traffic dataset} \cite{10.1145/2611567}, collected by Los Angeles County and \textbf{Windmill energy dataset} \cite{Fey/Lenssen/2019}, which is provided by Pytorch Geometric. Each dataset contains large-scale spatial-temporal observations for a long period of time and has been widely used in spatiotemporal data mining applications \cite{Fey/Lenssen/2019,10.1145/2611567,li2018dcrnn_traffic,ijcai2020-463,shao2021fadacs,DBLP:conf/ijcai/WuPLJZ19,Zhang_Chang_Meng_Xiang_Pan_2020}, with details below.

\textbf{Melbourne parking dataset} contains sensor data collected through 31 parking areas in Melbourne in 2017. We selected 10-month data ranging from 2017-01-01 to 2017-10-29 for the experiment. 

\textbf{METR-LA traffic dataset} consists of time-series data for traffic speed and volume measurements collected at different locations in Los Angeles. The dataset spans a period of several months and provides information at 5-minute intervals.

\textbf{Windmill energy dataset} contains the hourly energy output of windmills in a European country for more than two years. It contains 319 windmills in total and represents their relationship strength by the spatiotemporal graph.

% \small
\begin{table*}[!hbt]
\caption{ Average results on three datasets using different privacy budgets $ \epsilon $. We use bold numbers to indicate the best results for each row. Our proposed methods are marked by \textbf{*}. We use LSTM as the prediction model in this set of experiments. The other hyper-parameter setting follows the same mode which is discussed in Section \uppercase\expandafter{\romannumeral5}.C. The settings of DPGAN are followed to the \cite{xie2018differentially}. The Count rows stand for the total number of best results in each column. }
\begin{adjustbox}{width=0.85\textwidth, center}
    \centering
    % \footnotesize
    \begin{tabular}{c|c| | c c | c c | c c  }  
    \toprule
    \multicolumn{2}{c||}{Methods} &  \multicolumn{2}{c|}{DPGAN} & \multicolumn{2}{c|}{ST-DPGAN$^*$} & \multicolumn{2}{c}{ST-DPGAN-Attn$^*$} \\
    \midrule
    \multicolumn{2}{c||}{Metrics} & MSE & MAE & MSE & MAE & MSE & MAE   \\
    \midrule
    \multirow{5}{*}{ \rotatebox[origin=c]{90}{Parking} } 
        &  $\epsilon$=1    & 0.0504 & 0.1894 & 0.0578 & 0.2014 & \textbf{0.0349} & \textbf{0.1488} \\
        &  $\epsilon$=4    & 0.0296 & 0.1403 & 0.0481 & 0.1828 & \textbf{0.0315} & \textbf{0.1385} \\
        &  $\epsilon$=8    & 0.0247 & 0.1176 & 0.0157 & 0.0994 & \textbf{0.0134} & \textbf{0.0885} \\
        &  $\epsilon$=10   & 0.0166 & 0.0996 & \textbf{0.0118} & 0.0852 & 0.0129 & \textbf{0.0824} \\
        &  $\epsilon$=12   & 0.0136 & 0.0893 & \textbf{0.0087} & 0.0714 & \textbf{0.0087} & \textbf{0.0689} \\
        \midrule
    \multirow{5}{*}{ \rotatebox[origin=c]{90}{METR-LA} } 
        &  $\epsilon$=1    & 391.9801 & 16.8606 & 128.0023 & 8.2780 & \textbf{124.9095} & \textbf{8.1993} \\
        &  $\epsilon$=4    & 381.3257 & 16.6288 & 126.3622 & 8.1894 & \textbf{122.9242} & \textbf{8.0959} \\
        &  $\epsilon$=8    & 348.2818 & 15.7742 & 128.2958 & 8.2870 & \textbf{122.1717} & \textbf{8.0710} \\
        &  $\epsilon$=10   & 344.8490 & 15.7249 & 125.7851 & \textbf{8.1732} & \textbf{124.9988} & 8.1977 \\
        &  $\epsilon$=12   & 324.3259 & 15.2576 & 126.8095 & 8.2104 & \textbf{122.6244} & \textbf{8.0745} \\
        \midrule
    \multirow{5}{*}{ \rotatebox[origin=c]{90}{Windmill} } 
        &  $\epsilon$=1    & 0.1195 & 0.3167 & \textbf{0.0967} & \textbf{0.2688} & 0.1112 & 0.2864 \\
        &  $\epsilon$=4    & 0.1195 & 0.3167 & \textbf{0.0945} & \textbf{0.2616} & 0.1085 & 0.2828 \\
        &  $\epsilon$=8    & 0.1247 & 0.3233 & \textbf{0.0940} & \textbf{0.2618} & 0.1048 & 0.2780 \\
        &  $\epsilon$=10   & 0.1234 & 0.3217 & \textbf{0.0901} & \textbf{0.2561} & 0.1042 & 0.2773 \\
        &  $\epsilon$=12   & 0.1186 & 0.3156 & \textbf{0.0897} & \textbf{0.2560} & 0.0991 & 0.2703 \\
    \midrule
    \multicolumn{2}{c||}{Count} & 0 & 0 & 7 & 6 & 9 & 9 \\
    \bottomrule
    \end{tabular}
\end{adjustbox}
\label{tab:exp-eps}
\end{table*}

\subsection{Data Preprocessing}

In order to reduce computational complexity, we apply a threshold to filter the edges in each graph based on distance. This threshold helps determine which edges are significant and should be included in the subsequent analysis. The weighted adjacency matrix, denoted as W, can be constructed as follows:
\begin{equation}
w_{i j}=\left\{\begin{array}{lll}
\exp \left(-\frac{dis_{i j}^{2}}{\alpha^{2}}\right) &,& i \neq j \textbf{ and } \exp \left(-\frac{dis_{i j}^{2}}{\alpha^{2}}\right) \geq \beta \\
0  &,& \textbf{otherwise}
\end{array}\right.
\end{equation}
where $dis_{i j}$ indicates the distance between node i and node j, $\alpha$ and $\beta$ are two thresholds to control the sparsity of the weight matrix. The setting of the thresholds should be adjusted based on the datasets considering the computational cost.

\subsection{Experimental Settings}
Following the study conducted by Jordon et al. \cite{jordon2018pate}, we compare the generated data from the proposed model with real data and generated data from other methods. For real data, we employed a set of predictive models using real training and test data. For the generated data, the set of predictive models is trained on data generated by a generative model and evaluated using real test data. 

%this paper adopts two experimental settings. The first setting, referred to as "Setting A," involves training and evaluating a set of predictive models using real training and test data. In the second setting, known as "Setting B," a set of predictive models is trained on synthetic data generated by a generative model and evaluated using real test data. Furthermore, within Setting B, there are two sub-categories based on whether the generative model is trained under a differential privacy setting or not. Setting B without differential privacy is referred to as the "general case," while Setting B with differential privacy is referred to as the "privacy case."

In the ablation study, we introduce two proposed methods, namely ST-DPGAN and ST-DPGAN-Attn. In ST-DPGAN, we incorporate the transConv1D operation in the generator to facilitate information alignment while utilizing two spatiotemporal blocks from the discriminator side. In the case of ST-DPGAN-Attn, we replace the convolution layer within the spatiotemporal block with an attention block.

\textbf{Baselines}: The generation of spatiotemporal data using GANs with differential privacy is a relatively new and emerging area, as highlighted by Gao et al. \cite{gao2020generative}. Consequently, there is a limited number of established baselines that are specifically tailored to our research focus, particularly in the context of spatiotemporal data generation rather than image data. In this paper, we explore the application of spatiotemporal data generation tasks using alternative algorithms.

%We report the performance of DP-GAN \cite{xie2018differentially} and Wasserstein GAN (WGAN) \cite{arjovsky2017wasserstein} throughout the experiments.

We use DPGAN \cite{xie2018differentially} as an important baseline for data generation under a privacy setting. We use WGAN\cite{arjovsky2017wasserstein} as a baseline for data generation without differential privacy and also an upper bound for DPGAN. Compared to our proposed model, the selected baseline models do not contain the transConv1D module, graph embedding, and two-layer spatiotemporal attention blocks.

% We also adopt the spatiotemporal graph convolutional block from \cite{yu2017spatio} as part of the discriminator to implement a third baseline: ST-block-DPGAN. Compared to proposed model, it does not contain the transConv1D module and the spatiotemporal block is much more complicated.

% \textbf{Evaluation Metric}: We make use of a set of 11 predictive models in both \textit{setting A} and \textit{setting B}: Linear Regression \cite{scikit-learn}, Linear Support Vector Regressor \cite{smola2004tutorial}, Multiple Layer Perceptron \cite{hinton1990connectionist}, Stochastic Gradient Descent \cite{robbins1951stochastic}, Decision Tree \cite{loh2011classification}, Random Forest \cite{breiman2001random}, Gradient Boosting \cite{friedman2001greedy}, Bagging \cite{breiman1996bagging}, Ada Boosting \cite{drucker1997improving}, Light Gradient Boosting Machine \cite{10.5555/3294996.3295074} and Long Short-term Memory \cite{10.1162/neco.1997.9.8.1735}. Graph-based neural network is not chosen as regressor here because we would like to remove the potential bias in comparison since our model utilise the graph embedding and convolution structure. We select 70\%, 20\%, 10\% data for training, validation and testing respectively.

\textbf{Evaluation Metric}: We make use of the following predictive models in our experiments: Linear Regression\cite{mitchell2007machine}, Linear Support Vector Regressor\cite{mitchell2007machine}, Multiple Layer Perceptron\cite{rosenblatt1958perceptron}, Stochastic Gradient Descent\cite{mitchell2007machine}, Decision Tree\cite{mitchell2007machine}, Random Forest\cite{mitchell2007machine}, Gradient Boosting, Bagging, Ada Boosting \cite{scikit-learn}, Light Gradient Boosting Machine \cite{10.5555/3294996.3295074} and Long Short-term Memory~\cite{10.1162/neco.1997.9.8.1735}. The graph-based neural network is not chosen as a regressor here because we would like to remove the potential bias in comparison since our model utilizes the graph embedding and convolution structure. We select $70\%$, $20\%$, and $10\%$ data for training, validation, and testing, respectively.

We evaluate their performance using Mean Square Error (MSE) and Mean Absolute Error (MAE) because the downstream tasks are all regressions. The regression task can be defined as a typical time-series prediction problem, where we use a sliding window of size 6 across the data to make 3-step predictions. In the privacy aspect, we also use $\epsilon$ to measure the upper bounds of the privacy loss. A smaller privacy budget usually exhibits a better privacy guarantee over the data, while a higher value potentially brings more risks in privacy leakage.

\textbf{Setup}: All the models were compiled and tested using a single NVIDIA GeForce RTX 3090 24 GB GPU. To validate the performance of the proposed model with different privacy loss, we allocate five different levels of privacy budgets ($\epsilon$) when training on three datasets. In our experiment, we set the sampling probability $q$ = 0.01, noise level $\sigma$ = 2, and relaxation $\delta = 10^{-7}$.  We set the default epoch number as 400, and the training process stops if the epoch number exceeds 400. We also stop the training process for each DP-based model when the privacy budget is exhausted. Generators of all models are optimized using RMSProp optimizer. Discriminators of all DP-related models are optimized through DP-SGD optimizer \cite{abadi2016deep}. In comparison, we experiment with different optimizers to train the discriminators of all non-DP-related models and choose the best one. The batch size in all experiments is set to 10, and the global seed is set for experiment reproduction.

\subsection{Results and Analysis}

\begin{table*}[!hbt]
\caption{Ablation study results on three datasets. Our proposed methods are marked by \textbf{*}. The remaining columns show the MSE and MAE scores of predictive models under \textit{setting B}. Our models are marked by \textbf{*}. The lower the score value, the better the data quality. The best-performing entry under \textit{setting B} is in bold. We choose the privacy budget $\epsilon$ as 12 for the DP-based model to compare with other non-DP contained algorithms. The Average row stands for the average loss under different prediction algorithms. The Count rows represent the total number of best results in each column. }
\begin{adjustbox}{width=1.05\textwidth, center}
    \centering
    \begin{tabular}{c|c|| c c | c c | c c | c c | c c}  
    \toprule
    \multicolumn{2}{c||}{Methods} & \multicolumn{2}{c|}{ST-DPGAN-TST}  & \multicolumn{2}{c|}{ST-DPGAN-TSTS} & \multicolumn{2}{c|}{ST-DPGAN-GE}  & \multicolumn{2}{c|}{ST-DPGAN-TC} & \multicolumn{2}{c}{ST-DPGAN-Attn$^*$} \\
    \midrule
    \multicolumn{2}{c||}{Metrics} & MSE & MAE & MSE & MAE & MSE & MAE & MSE & MAE & MSE & MAE \\
    
    \midrule\midrule
    \multirow{11}{*}{ \rotatebox[origin=c]{90}{Parking} } 
    & LR                       & 0.0070 & 0.0634     & 0.0071 & 0.0662     & \textbf{0.0066} & \textbf{0.0595}     & 0.0381 & 0.1535 & 0.0068 & 0.0643 \\
    & Linear Regression        & 0.0070 & 0.0607     & 0.0068 & 0.0603     & 0.0066 & 0.0589     & 0.0575 & 0.1934 & \textbf{0.0062} & \textbf{0.0576} \\
    & MLP                      & 0.0079 & 0.0684     & 0.0083 & 0.0713     & 0.0105 & 0.0760     & 0.0369 & 0.1507 & \textbf{0.0068} & \textbf{0.0645} \\
    & SGD                      & 0.0071 & 0.0638     & 0.0072 & 0.0664     & \textbf{0.0066} & \textbf{0.0595}     & 0.0391 & 0.1551 & 0.0069 & 0.0646 \\
    & Decision Tree            & 0.0490 & 0.1661     & 0.0267 & 0.1170     & 0.0203 & 0.0961     & 0.0696 & 0.1997 & \textbf{0.0136} & \textbf{0.0828} \\
    & Random Forest            & 0.0153 & 0.0995     & 0.0129 & 0.0924     & 0.0099 & 0.0791     & 0.0530 & 0.1804 & \textbf{0.0069} & \textbf{0.0613} \\
    & Gradient Boosting        & 0.0105 & 0.0790     & 0.0118 & 0.0863     & 0.0081 & 0.0698     & 0.0382 & 0.1534 & \textbf{0.0071} & \textbf{0.0630} \\
    & Bagging                  & 0.0125 & 0.0840     & 0.0154 & 0.0957     & 0.0109 & 0.0771     & 0.0267 & 0.1268 & \textbf{0.0076} & \textbf{0.0638} \\
    & Ada Boosting             & 0.0152 & 0.1003     & 0.0159 & 0.1043     & 0.0100 & 0.0807     & 0.0395 & 0.1584 & \textbf{0.0096} & \textbf{0.0774} \\
    & LGBM                     & 0.0104 & 0.0787     & 0.0100 & 0.0797     & 0.0078 & 0.0683     & 0.0410 & 0.1589 & \textbf{0.0065} & \textbf{0.0603} \\
    & LSTM                     & 0.0148 & 0.0982     & 0.0158 & 0.1015     & 0.0155 & 0.1006     & 0.0632 & 0.1953 & \textbf{0.0087} & \textbf{0.0689} \\
    
    \midrule
    \multicolumn{2}{c||}{Average}  & 0.0142 & 0.0875    & 0.0125 & 0.0855    & 0.0102 & 0.0751    & 0.0457 & 0.1660  & 0.0079 & 0.0662   \\
    \midrule\midrule
    \multirow{11}{*}{ \rotatebox[origin=c]{90}{METR-LA } } 
    
    & LR                       & 98.4122 & 7.0480       & 98.4824 & 7.0603       & 97.1667 & 6.8635      & 98.4795 & 7.0600     & \textbf{97.0645}  & \textbf{6.7750}  \\
    & Linear Regression        & 98.7342 & 7.0686       & 98.4922 & 7.0887       & 97.7149 & \textbf{6.6708}      & 98.5268 & 7.0911     & \textbf{97.2597}  & 6.6976  \\
    & MLP                      & 126.1910 & 9.4255      & 120.8335 & 9.0971      & 100.3331 & 7.1784     & 120.4894 & 9.0752    & \textbf{97.6307}  & \textbf{6.8797}  \\
    & SGD                      & 106.1432 & 7.9692      & 109.7475 & 8.2786      & \textbf{96.9715} & 6.8293      & 109.8588 & 8.2877    & 97.0175  & \textbf{6.7621}  \\
    & Decision Tree            & 524.1221 & 21.1745     & 648.9779 & 23.6976     & 189.5849 & 9.7803     & 442.8109 & 19.3704   & \textbf{142.6260}  & \textbf{8.0813}  \\
    & Random Forest            & 544.3726 & 21.6705     & 555.5890 & 21.9025     & 135.0329 & 8.2709     & 554.9820 & 21.8896   & \textbf{108.7084}  & \textbf{6.8320}  \\
    & Gradient Boosting        & 456.9711 & 19.7746     & 480.9894 & 20.2982     & 104.7566 & 7.0779     & 473.9832 & 20.1490   & \textbf{102.2363}  & \textbf{6.7753}  \\
    & Bagging                  & 479.9638 & 20.2789     & 460.4139 & 19.8466     & 116.7623 & 7.6960     & 477.7846 & 20.1838   & \textbf{113.0507}  & \textbf{7.0063}  \\
    & Ada Boosting             & 456.1204 & 19.7248     & 475.6228 & 20.1578     & 170.1427 & 11.4454    & 484.3269 & 20.3551   & \textbf{107.5073}  & \textbf{8.0823} \\
    & LGBM                     & 469.6165 & 20.0458     & 487.3716 & 20.4343     & 102.7120 & 7.0131     & 489.6298 & 20.4841   & \textbf{102.0630}  & \textbf{6.9834}  \\
    & LSTM                     & 341.9394 & 17.0463     & 354.2284 & 17.3683     & 183.4939 & 12.0563    & 354.1376 & 17.3660   & \textbf{122.6244}  & \textbf{8.0745} \\
    
    \midrule
    \multicolumn{2}{c||}{Average}  & 336.5988 & 15.5661   & 353.7044 & 15.9300   & 126.7882 & 8.2620   & 336.8190 & 15.5738  & 107.9808 & 7.1772  \\
    \midrule\midrule
    \multirow{11}{*}{ \rotatebox[origin=c]{90}{Windmill } } 
    & LR                       & 0.1098 & 0.3054     & 0.0608 & 0.1993     & 0.1272 & 0.3291     & 0.1176 & 0.3163    & \textbf{0.0581} & \textbf{0.2022} \\
    & Linear Regression        & 0.1031 & 0.2955     & \textbf{0.0608} & 0.1983     & 0.1275 & 0.3295     & 0.0995 & 0.2901    & 0.0771 & \textbf{0.1751} \\
    & MLP                      & 0.1052 & 0.2987     & 0.0615 & 0.2053     & 0.1317 & 0.3350     & 0.1278 & 0.3298    & \textbf{0.0578} & \textbf{0.2011} \\
    & SGD                      & 0.1097 & 0.3052     & 0.0692 & 0.2326     & 0.1277 & 0.3297     & 0.0928 & 0.2793    & \textbf{0.0582} & \textbf{0.2028} \\
    & Decision Tree            & \textbf{0.1130} & 0.2923     & 0.1247 & 0.3254     & 0.2195 & 0.3785     & 0.1460 & 0.3495    & 0.1353 & \textbf{0.2702} \\
    & Random Forest            & 0.1229 & 0.3233     & 0.1254 & 0.3267     & 0.1282 & 0.3304     & 0.1253 & 0.3265    & \textbf{0.0637} & \textbf{0.2086} \\
    & Gradient Boosting        & 0.1138 & 0.3102     & 0.1246 & 0.3257     & 0.1289 & 0.3310     & 0.1478 & 0.3534    & \textbf{0.0590} & \textbf{0.2063} \\
    & Bagging                  & 0.1010 & 0.2863     & 0.1236 & 0.3243     & 0.1397 & 0.3363     & 0.1281 & 0.3290    & \textbf{0.0699} & \textbf{0.2137} \\
    & Ada Boosting             & 0.1223 & 0.3225     & 0.1226 & 0.3230     & 0.1283 & 0.3306     & 0.1221 & 0.3224    & \textbf{0.0604} & \textbf{0.2075} \\
    & LGBM                     & 0.1230 & 0.3235     & 0.1236 & 0.3243     & 0.1280 & 0.3302     & 0.1263 & 0.3279    & \textbf{0.0592} & \textbf{0.2069} \\
    & LSTM                     & 0.0942 & \textbf{0.2026}     & 0.0932 & 0.2120     & 0.0788 & 0.2529     & \textbf{0.0727} & 0.2397    & 0.0991 & 0.2703 \\
    
    \midrule
    \multicolumn{2}{c||}{Average}  & 0.1107 & 0.2968   & 0.0991 & 0.2724   & 0.1332 & 0.3284   & 0.1187 & 0.3149  & 0.0725  & 0.2149  \\
    \midrule\midrule
    \multicolumn{2}{c||}{Count} & 1 & 1 & 1 & 0 & 3 & 3 & 0 & 0 & 27 & 29 \\
    \bottomrule
    
    \end{tabular}
\end{adjustbox}
\label{tab:ablation}
\end{table*}

This section discusses the experimental results obtained under different settings. Table \ref{tab:exp-main} presents the results of our proposed model and baselines on the three mentioned datasets. We trained a Wasserstein GAN (WGAN) to compare its performance with our proposed method in the above tasks. The training and network settings for the WGAN follow the algorithm outlined in Arjovsky et al. \cite{arjovsky2017wasserstein}.

In most cases, the data generated by ST-DPGAN demonstrate lower loss and outperform the WGAN in regression tasks. Moreover, incorporating the attention mechanism in ST-DPGAN-Attn further enhances the data quality by leveraging the inherent spatial and temporal relationships among nodes. These findings validate the efficacy of our methodology, as presented in Section \uppercase\expandafter{\romannumeral4}. B.

% \addtolength\tabcolsep{2pt}

Table \ref{tab:exp-eps} presents the high-level results of our proposed model under a privacy-preserving setting. We evaluate the quality of the generated data under different privacy budgets defined by $\epsilon$. In our experiments, we compare our methods with DPGAN, which follows the same settings described in Xie et al. \cite{xie2018differentially}. The results indicate that, with the same privacy budgets (represented by $\epsilon$), ST-DPGAN and ST-DPGAN-Attn are more effective than DPGAN in capturing and recovering spatiotemporal information. Furthermore, this performance superiority is consistently maintained across different privacy budgets, suggesting that our method is stable and robust to varying $\epsilon$ values. Overall, ST-DPGAN and ST-DPGAN-Attn have demonstrated the potential to generate high-quality data with privacy protection, surpassing existing models.

The significant gap observed between the baseline models and our proposed model can be primarily attributed to the inclusion of the transposed 1D convolutional layer and graph embedding, which facilitate the generation of high-quality spatiotemporal data. An in-depth analysis of these components is provided in the ablation study section.

By comparing the real data and the generated data, we can further explore the performance of ST-DPGAN in terms of noise resistance and the potential for improvement. From the table, it can be observed that a smaller privacy budget leads to relatively higher loss, while a larger privacy budget can potentially increase the utility of generated data but comes with a reduced guarantee of privacy. Additionally, we find that the utility of the generated data decreases in a non-linear manner. The privacy loss represented by $\epsilon$ is not directly proportional to the changes in MSE and MAE. Therefore, certain high-cost-effective points can achieve good performance with acceptable privacy preservation.

\subsection{Ablation Study}
We conduct further experiments to study the significance of model components as an ablation study. In detail, we want to find out: (a) how much the transposed 1D convolutional layer in the generator and graph embedding in the discriminator help our task in terms of data quality, and (b) whether increasing the number of spatial and temporal blocks could help produce data of higher quality.

We specifically create 4 variants based on ST-DPGAN: \textit{ST-DPGAN-TC}, where the transposed 1D convolutional layer is removed from generator; and \textit{ST-DPGAN-GE}, where graph embedding is removed from discriminator are used for study (a); \textit{ST-DPGAN-TST}, where an additional temporal block is added to form the sandwich structure from \cite{yu2017spatio}, and \textit{ST-DPGAN-TSTS} where additional temporal and spatial blocks are appended for study (b), which have four temporal and spatial blocks in total.

The ablation study results are shown in table \ref{tab:ablation}. We can observe a serious performance degeneration for \textit{ST-DPGAN-TC} and \textit{ST-DPGAN-GE} after the transposed 1D convolutional layer and graph embedding are removed. It proves the significance of these components when generating high-quality spatial-temporal data. On the other hand, increasing the number of spatial and temporal blocks does not bring any performance gain. We attribute this to the instability and difficulty in training brought by increased spatial and temporal blocks, and it has empirically validated our claim in section \uppercase\expandafter{\romannumeral4}.B.

\section{Conclusion}
In this paper, we introduced ST-DPGAN, a novel generative adversarial model for creating differential privacy-compliant spatiotemporal graph data. Our experiments show that ST-DPGAN surpasses other models in generating high-quality, privacy-preserving spatiotemporal graph data on three real-world datasets. The ST-DPGAN model promises wide-ranging applications, offering comprehensive data for unbiased neural network training and robustness against abnormal data. Its capability to handle diverse data sources allows for its application beyond specific domains like traffic and weather prediction, potentially improving various aspects of daily life. Crucially, ST-DPGAN addresses significant privacy concerns and legal risks in spatiotemporal data analysis, marking a significant step in responsible data handling. ST-DPGAN addresses the privacy issue in spatiotemporal data generation. 

\bibliographystyle{plain}
\balance
\bibliography{mybibliography}

\end{document}